\documentclass[10pt]{article}
\usepackage[preprint]{tmlr}

\usepackage{hyperref}
\usepackage{url}
\usepackage{graphicx}
\usepackage{amsmath, amsfonts, amsthm}
\usepackage{tabularx}
\usepackage{arydshln}
\usepackage{booktabs}

\newtheorem{proposition}{Proposition}
\newtheorem{definition}{Definition}

\DeclareMathOperator{\BER}{BER}
\DeclareMathOperator{\NCB}{NCB}
\DeclareMathOperator{\balance}{balance}

\DeclareMathOperator{\hsic}{HSIC}
\DeclareMathOperator{\tr}{tr}
\DeclareMathOperator{\cka}{CKA}
\DeclareMathOperator{\softmax}{softmax}

\DeclareMathOperator{\Xsyn}{\tilde{X}}
\DeclareMathOperator{\Ssyn}{\tilde{S}}
\newcommand{\tran}{^{\mkern-1.5mu\mathsf{T}}}

\newcommand{\tabARGN}{TabularARGN\ }
\newcommand{\tabfairgan}{TabFairGAN\ }
\newcommand{\prefair}{PreFair\ }
\newcommand{\fldgm}{FLDGM}
\newcommand{\ancb}{\mbox{A-NCB}}

\newcolumntype{Y}{>{\centering\arraybackslash}X}
\newcolumntype{L}[1]{>{\raggedright\arraybackslash\hspace{0pt}}p{#1}}
\newcolumntype{M}[1]{>{\centering\arraybackslash}m{#1}}
\newcommand{\Checkmark}{\scalebox{2}{\checkmark}}

\begin{document}
\title{\raggedright Achieving Hilbert-Schmidt Independence Under Rényi Differential Privacy for Fair and Private Data Generation}
\author{%
    \name Tobias Hyrup \email hyrup@imada.sdu.dk\\
    \addr Department of Mathematics and Computer Science\\
    University of Southern Denmark\\
    Odense, Denmark
    \AND
    \name Emmanouil Panagiotou \email emmanouil.panagiotou@fu-berlin.de\\
    \addr Department of Mathematics and Computer Science\\
    Freie Universität Berlin\\
    Berlin, Germany
    \AND
    \name Arjun Roy \email arjun.roy@unibw.de\\
    \addr Faculty for Informatik\\
    Universität der Bundeswehr München\\
    Munich, Germany
    \AND
    \name Arthur Zimek \email zimek@imada.sdu.dk\\
    \addr Department of Mathematics and Computer Science\\
    University of Southern Denmark\\
    Odense, Denmark
    \AND
    \name Eirini Ntoutsi \email eirini.ntoutsi@unibw.de\\
    \addr Faculty for Informatik\\
    Universität der Bundeswehr München\\
    Munich, Germany
    \AND
    \name Peter Schneider-Kamp \email petersk@imada.sdu.dk\\
    \addr Department of Mathematics and Computer Science\\
    University of Southern Denmark\\
    Odense, Denmark
}

\maketitle

\begin{abstract}
As privacy regulations such as the GDPR and HIPAA and responsibility frameworks for artificial intelligence such as the AI Act gain traction, the ethical and responsible use of real-world data faces increasing constraints. Synthetic data generation has emerged as a promising solution to risk-aware data sharing and model development, particularly for tabular datasets that are foundational to sensitive domains such as healthcare. To address both privacy and fairness concerns in this setting, we propose FLIP (Fair Latent Intervention under Privacy guarantees), a transformer-based variational autoencoder augmented with latent diffusion to generate heterogeneous tabular data. Unlike the typical setup in fairness-aware data generation, we assume a task-agnostic setup, not reliant on a fixed, defined downstream task, thus offering broader applicability. To ensure privacy, FLIP employs Rényi differential privacy (RDP) constraints during training and addresses fairness in the input space with RDP-compatible balanced sampling that accounts for group-specific noise levels across multiple sampling rates. In the latent space, we promote fairness by aligning neuron activation patterns across protected groups using Centered Kernel Alignment (CKA), a similarity measure extending the Hilbert-Schmidt Independence Criterion (HSIC). This alignment encourages statistical independence between latent representations and the protected feature. Empirical results demonstrate that FLIP effectively provides significant fairness improvements for task-agnostic fairness and across diverse downstream tasks under differential privacy constraints.

\end{abstract}

\section{Introduction}
With a continuous increase in digital dependence, the amount of data being created or recorded increases with it. Much of these data can reveal valuable insights about populations and individuals while forming a foundation for automated, individualized decision-making, having a significant social impact~\citep{barocas_selbst2016}. However, if data processing is performed without explicitly taking into account fairness and privacy considerations, this neglect may propagate to subsequent inference and decision-making processes, leading to biased outcomes and privacy leakage. Regulatory frameworks such as the EU General Data Protection Regulations (GDPR)~\citep{GDPR2016}, the Health Insurance Portability and Accountability Act (HIPAA)~\citep{OfficeforCivilRightsOCR2012GuidanceRule}, and the EU Artificial Intelligence Act (AI Act)~\citep{aiact2021} increasingly emphasize the social responsibilities of data and AI systems. In the wake of this, it becomes imperative to design models that are not only performant, but also fair and privacy-preserving~\citep{kiran2025-fair-private-synthetic}.

In this work, we present Fair Latent Intervention under Privacy guarantees (FLIP), illustrated in Figure~\ref{fig:intro}, which performs task-agnostic disentanglement of the protected attribute under Rényi differential privacy. By adopting a two-phase training procedure, FLIP is first trained to learn a high-quality representation of the true data, followed by a secondary phase that disentangles the protected attribute from the training data.

In contrast, prior fairness-aware work on synthetic data generation focus overwhelmingly on task-specific downstream tasks, limiting the applicability of synthetic data to that particular task~\citep{make4020022, tiwald2025tabularargn, Pujol2023}. Instead, we design FLIP independently of any downstream task, thus broadening the application range of synthetic data. In addition, we provide an RDP-compatible balanced Poisson sampling scheme that provides group-wise noise levels while ensuring equal representation of protected groups in each batch.

\begin{figure}[tb]
    \centering
    \includegraphics[width=0.6\linewidth]{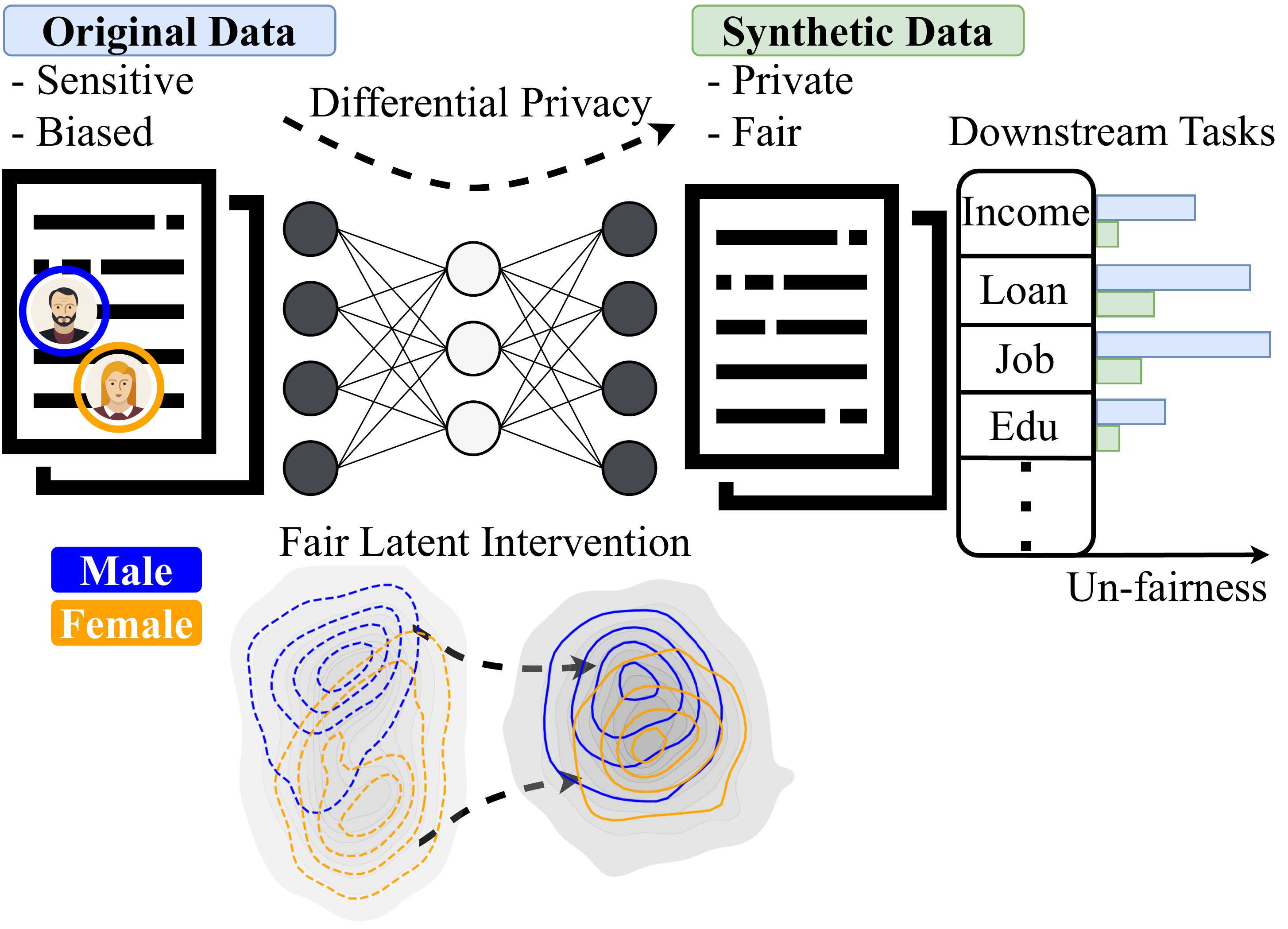}
    \caption{FLIP: Fair Latent Intervention under Privacy Guarantees.}
    \label{fig:intro}
\end{figure}

\section{Related Work}\label{sec:related_Work}
Recent advancements in generative machine learning for tabular data have led to various methods that focus on improving the utility of synthetic data, but also on fairness and privacy. These lines of work are closely related to our approach, which aims to address all these aspects in a unified framework. The following sections cover related work. Table~\ref{tab:related_work_comparison} presents the key differences for selected approaches compared to our method.

\subsection{Fair Generative Models}
Most fairness-aware generative approaches focus on a single, predefined downstream task and follow a two-stage training procedure: first, learning a strong data generator, and second, incorporating fairness objectives. Early work in this space typically employed GAN-based architectures. For example, \tabfairgan\citep{make4020022} introduces a debiasing loss applied post hoc to the generator, while CuTS \citep{pmlr-v235-vero24a} proposes a customizable generator architecture tailored to fairness constraints. DECAF \citep{NEURIPS2021_ba9fab00}, on the other hand, leverages a causal framework to eliminate discriminatory causal links between sensitive attributes and the target during generation. However, subsequent studies \citep{wang2022replication} have critiqued DECAF for poor performance in mitigating unfairness. Nonetheless, a recent approach that takes a more abstract view of fairness is the Fair Latent Diffusion Generative Model (\fldgm) \citep{ramachandranpillai2023fair}, which incorporates fairness by removing sensitive attribute information in the latent space. This aligns with our view of fairness as an intrinsic data property. However, \fldgm\ still evaluates fairness using task-specific metrics, limiting its applicability to broader, task-agnostic scenarios.

\subsection{Differentially Private Generative Models}
Differential privacy (DP)~\citep{dwork2006}, provides formal guarantees for the generative model; a property that is passed on to synthetic data due to the post-processing property of DP~\citep{dwork2014}. DP-based deep generative models have generally been dominated by generative adversarial models~\citep{TORFI2022485, Ma20234838, 10.1007/978-3-031-09342-5_17, yoon2018pategan} due to the general success of the architecture and the fact that only the discriminator requires DP training. DP-GAN~\citep{xie2018dpgan}, an early DP-based GAN model, uses the moment accountant mechanism for privacy composition, while more recent works often apply Rényi differential privacy (RDP) such as RDP-GAN~\citep{Ma20234838}, DP-CTGAN \citep{10.1007/978-3-031-09342-5_17} and RDP-CGAN \citep{TORFI2022485}. The latter also addresses difficulties with heterogeneous tabular data using a convolution variational autoencoder.

\subsection{Fair and Private Generative Models}
Another method based on a causal model and marginals, \prefair\citep{Pujol2023} optimizes for fairness, given a target and protected attribute, while using a differentially private (DP) algorithm. More recently, state-of-the-art methods for generating synthetic tabular data have shifted toward transformer-based autoregressive modeling. Notably, \tabARGN\citep{tiwald2025tabularargn} employs a deep autoregressive generator that explicitly optimizes for fair sampling only when generating the downstream target attribute. The method also supports DP training, although, rather than computing appropriate noise levels for a set number of iterations, the specified privacy budget acts as an early stopping criterion.

\begin{table}[tbp]
    \renewcommand{\arraystretch}{1.5}
    \centering
    \footnotesize
    \caption{Comparison of this work with existing literature.}
    \label{tab:related_work_comparison}
    \begin{tabular}{L{30mm}L{21mm}M{15mm}M{13mm}M{20mm}M{25mm}}
         \toprule
         \textbf{Publication} & \textbf{Model} & \textbf{Fairness} & \textbf{Privacy} & \textbf{Task-agnostic} & \textbf{Task-agnostic evaluation}\\
         \midrule 
         \citet{make4020022} & \tabfairgan & \Checkmark & & & \\
         \citet{Pujol2023} & \prefair & \Checkmark & \Checkmark & & \\
         \citet{ramachandranpillai2023fair} & \fldgm & \Checkmark & & \Checkmark & \\
         \citet{tiwald2025tabularargn} & \tabARGN & \Checkmark & \Checkmark & & \\[1ex]
         \cdashline{1-6}\\[-3ex]
         This work & FLIP & \Checkmark & \Checkmark & \Checkmark & \Checkmark   \\
         \bottomrule
    \end{tabular}
\end{table}

\section{Background}\label{sec:background}
This section outlines the foundational concepts of fairness and privacy, providing the theoretical background and definitions used in our work.

\subsection{Common Notation}
Here, we introduce a standard set of notation for consistent use throughout this work. Supplementary notations are introduced as needed in later sections. Let $\mathcal{D} = (X, S)$ denote the complete dataset consisting of both the primary feature matrix $X \in \mathbb{R}^{n \times k}$ and the binary protected attribute $S \in \{0, 1\}^n$, where $n$ and $k$ are the number of samples and features, respectively. Let $\mathcal{S} = \{0, 1\}$ denote the set of possible values for $S$. 

Let $A = \{a_1, a_2, \dots, a_k\}$ denote the set of attributes in $X$. We then define column-wise partitions of matrix $X$ on a subsets of features $A' \subseteq A$ as $X^{(A')}$. Likewise, we define row-wise partitions of matrix $X$ as $X_{(S=s)}=\{x_i \in X \mid S=s\}$, where $S$ is a feature of $X$.

We define $P_z$ as a simple prior probability distribution from which i.i.d.\ samples can be drawn, then a generative model $G_\theta : z \rightarrow \tilde{x},\ z \sim P_z$ generates synthetic data w.r.t.\ learned parameters $\theta$. The resulting synthetic dataset is denoted by $\tilde{\mathcal{D}} = (\Xsyn, \Ssyn)$.

\subsection{Fairness}
In this work, we focus on \emph{group} fairness, which can be defined largely as a property of an algorithm $\mathcal{A}$ such that there is no negative bias to the detriment of certain groups. This definition of fairness requires specifying a protected attribute; an attribute whose values denote membership in groups subject to protection against discrimination~\citep{le2022survey}. Traditionally, machine learning research has focused on fairness for downstream tasks that target a predefined target attribute $y \in Y$, and it can be evaluated for inference bias in $Y$ with respect to the membership of the protected group. Consequently, such fairness definitions are restricted to a predictive algorithm with a single target attribute. 

\subsubsection{Fairness for Downstream Tasks}

Many notions of algorithmic fairness exist in the domain. A common approach is Fairness Through Unawareness (FTU), which assumes fairness if the algorithm does not explicitly use the protected attribute~\citep{castelnovo2022clarification, NEURIPS2021_ba9fab00}. However, this ignores indirect influence through correlated features. Statistical Parity (SP) provides a stricter notion by requiring that the predicted outcomes be independent of group membership, i.e., $\Pr(\hat{Y} | S = s) = \Pr(\hat{Y} | S = s')$ for all $s, s' \in S$.

\subsubsection{Fairness for Synthetic Data}
\label{sec:fair_synth_dat}
Fairness in the context of synthetic data requires a distinct perspective, as it is no longer solely an attribute of an algorithm, such as a classification model, but also an inherent property of the data themselves. Nevertheless, many works consider the fairness level of their generative approaches through the lens of a downstream model~\citep{ramachandranpillai2023fair, fairgan2022, make4020022, NEURIPS2021_ba9fab00, pmlr-v235-vero24a, Choi_Park_Kim_Lee_Park_2024}. The rationale is that if a downstream model trained on synthetic data without explicit fairness interventions does not discriminate between protected groups, then the synthetic data are considered fair with respect to the protected and target attributes. This evaluation approach, which we refer to as \emph{Task Fairness} (TF), assesses whether a downstream model, trained on synthetic data, satisfies a given fairness criterion $\mathcal{U}(X, Y)$.

Nonetheless, TF assumes that fairness is evaluated concerning a single, predefined target task. Indeed, while many datasets are curated with a specific task in mind, this is not always the case, especially in real-world scenarios. Synthetic datasets are often created for broad distribution and open-ended use, where the final application is unknown. In this more abstract, but also more realistic scenario, the assumption underlying TF becomes too restrictive. Consider, for example, a financial dataset originally intended for predicting loan eligibility. If synthetic data are generated and evaluated for fairness with respect to the outcome variable \emph{loan}, then no fairness guarantees can be made when the data are used to predict \emph{income level}, even though it is a perfectly valid task for this dataset.

In our work, we focus on the more realistic, more abstract, but also more difficult problem: evaluating fairness as a property inherent to the generated data itself, rather than relative to a specific downstream task. To this end, we consider two evaluation approaches that measure fairness as a property only of the generated data and the protected attribute. Specifically, $\epsilon$-fairness measures the predictability of the protected attribute given all other attributes, while cluster fairness evaluates how well the protected attribute aligns with the underlying cluster structure.

\begin{definition}
\label{def:epsilon-fairness}
    ($\epsilon$-fairness~\citep{fairgan2022, Feldman2015}). A generative mode $G_{\theta}$ is said to be $\epsilon$-fair iff for any classification model $f : \Xsyn \rightarrow \Ssyn$
    \[
        \BER \left(f(\Xsyn), \Ssyn\right) > \epsilon,
    \]
    where the balanced error rate (BER) is defined as
    \begin{equation}    
        \BER = \frac{1}{2}\Big(\Pr[f(\Xsyn) = 0 \mid \Ssyn = 1] \\ 
        + \Pr[f(\Xsyn) = 1 \mid \Ssyn = 0]\Big)
    \end{equation}
\end{definition}

Similar to measuring the predictability of the protected attribute, the clusterability can also be evaluated by computing the cluster balance~\citep{NIPS2017_978fce5b, LeQuy2021}. We extend the notion of cluster balance to a normalized version (Normalized Cluster Balanced (NCB)) to accommodate imbalanced protected groups.

\begin{definition}
\label{def:cluster-fairness}
    (Cluster Fairness). Given a clustering $C =\{C_1, C_2\}$, $f : \Xsyn \rightarrow C$, a generative model $G_{\theta}$ is cluster fair w.r.t.\ a threshold $\epsilon$ iff:
    \[
        \NCB\left(f(\Xsyn), \Ssyn\right) > \epsilon,
    \]
    where the normalized cluster balance (NCB) is computed as
    \begin{align*}
        \balance(c, S) &= \min\left(\frac{\Pr[x \in c | S = 0]}{\Pr[x \in c | S = 1]}, \frac{\Pr[x \in c | S = 1]}{\Pr[x \in c | S = 0]} \right) \\
        \NCB(C, S) &= \min_{c \in C} \balance(c, S)
    \end{align*}
\end{definition}

Given these definitions, we are now ready to define fairness from a task-agnostic perspective with appropriate evaluation measures.

\begin{definition} \label{def:base_fairness}
    (Disentangled Fairness (DF)). A generative model $G_{\theta}$ satisfies disentangled fairness iff 
    \begin{equation*}
        \forall \tilde{\mathcal{D}} \sim G_{\theta}(Z),\ (\Xsyn, \Ssyn) \in \tilde{\mathcal{D}},\ \Xsyn \perp\!\!\!\perp \Ssyn.
    \end{equation*}
\end{definition}

Definition~\ref{def:base_fairness} defines fairness as the complete disentanglement of the protected attribute from the data. In such a dataset, any underlying structure related to the protected attribute has been erased, and the attribute no longer carries any informative value with respect to the other attributes. The degree to which this has been achieved can be measured by $\epsilon$-fairness and cluster fairness. As we do not consider any attribute a target, \textit{task fairness} measures alone cannot capture the disentanglement of $S$. We argue that while this definition may not constitute an optimal fairness objective when the downstream tasks are known, it does represent the likely scenario where the downstream task or tasks are unknown. Effectively, Disentangled Fairness considers all non-protected attributes as targets, treating them all equally.

\subsection{Differential Privacy} \label{sec:privacy:dp}
Differential privacy using the Gaussian mechanism provides guarantees that a single individual can at most influence the output distribution of an algorithm by a factor of $\exp(\varepsilon)$ with a probability $\delta$ that it does not hold, thus providing privacy guarantees for the algorithm itself~\citep{dwork2006}. Formally, DP is defined in Definition~\ref{def:dp}.

\begin{definition} \label{def:dp}
    (($\epsilon, \delta$)-Differential Privacy~\citep{dwork2014, ilya2017}). Let $\mathcal{A} : \mathcal{D} \rightarrow \mathcal{R}$ be a randomized algorithm. $\mathcal{A}$ satisfies ($\varepsilon, \delta$)-DP if for all adjacent datasets $\mathcal{D}_{1}, \mathcal{D}_{2}$ and $J \subset \mathcal{R}$
    \[
        \Pr[\mathcal{A}(\mathcal{D}_{1}) \in J] \leq \exp(\varepsilon)\Pr[\mathcal{A}(\mathcal{D}_{2}) \in J] + \delta,
    \]
\end{definition}

In practice, carefully calibrated noise can be added to an algorithm, essentially obfuscating the true contribution of individuals. Using the Gaussian mechanism, let $c$ be the $\ell_2$-sensitivity of the data, that is, the largest absolute difference in the training data, let $\sigma$ be the standard deviation, and let $f: X \rightarrow \mathbb{R}$, then noise is added as~\citep{10.1145/2976749.2978318}:
\begin{equation} \label{eq:dp_noise}
    \mathcal{A}(\mathcal{D}) := f(\mathcal{D}) + \mathcal{N}(0, (c\sigma)^2)
\end{equation}

DP stochastic gradient descent (DP-SGD) offers an efficient approach to provide DP guarantees to neural networks~\citep{10.1145/2976749.2978318}. Under this optimization scheme, noise is added to the per-sample gradients, while controlling the sensitivity by clipping the gradient norms. For efficiency, gradient perturbation is performed on the aggregated mini-batch gradients. In addition, the fact that not all training samples are used during each batch can be exploited to amplify privacy, which is typically performed using Poisson sampling without replacement and a given sample rate~\citep{10.1145/2976749.2978318, pmlr-v97-zhu19c}. 

Furthermore, the privacy loss can be tightly composed over multiple iterations using the Rényi divergence. Formally, Rényi DP (RDP) is defined in Definition~\ref{def:rdp} and offers nice conversion to the standard $(\varepsilon, \delta)$-DP definition as \mbox{$\left(\varepsilon + \frac{\log 1/\delta}{\alpha -1 }, \delta\right)$-DP}~\citep{ilya2017}.

\begin{definition} \label{def:rdp}
    (Rényi Differential Privacy (RDP)~\citep{ilya2017}). Let $D_{\alpha}(\cdot \lVert \cdot)$ define the Rényi divergence with degree $\alpha$. A randomized algorithm $\mathcal{A}$ is said to satisfy $(\alpha, \varepsilon_{\mathcal{A}}(\alpha))$-RDP if for any adjacent datasets $\mathcal{D}_1, \mathcal{D}_2$
    \[
        D_{\alpha}\left(\mathcal{A}(\mathcal{D}_1) \parallel \mathcal{A}(\mathcal{D}_2) \right) \leq \varepsilon_\mathcal{A}(\alpha).
    \]
\end{definition}

Combining DP-SGD with an RDP accountant allows efficient training under tight compositions, making differentially private neural networks viable.

\subsection{Integration Challenges}
Modeling fair and private representations of real data while maintaining utility is highly challenging. While privacy-oriented perturbations aim at obfuscating the learned representation enough to provide a privacy cover, fairness interventions through attribute disentanglement actively induce bias-mitigating distributional modulations. Consequently, the learned distribution will intentionally be dissimilar from the true distribution, causing subsequent fidelity and utility evaluations to not reflect the true applicability to real-world settings. This poses a challenge within the privacy, fairness, and quality triad, as existing methods struggle to effectively assess how well synthetic data represent a reference distribution that is truly private, fair, and useful, but typically unknown.

\section{FLIP: Fair Latent Intervention Under Privacy Guarantees}\label{sec:methods}
Exploring the interaction between fairness and privacy for tabular data requires a flexible architecture that allows for targeted training perturbations while maintaining high-quality and diverse outputs. Additionally, generating heterogeneous tabular data is challenging, as it requires capturing diverse data types and intricate relationships and constraints within the data. To this end, we propose \textbf{Fair Latent Intervention Under Privacy Guarantees (FLIP)}, a transformer-based variational autoencoder with latent diffusion that allows perturbation of the learned representations while mapping the latent space to a simple prior through a diffusion process. The latent space effectively acts as a joint encoding of the heterogeneous data, allowing for multivariate, fairness-informed perturbation of the learned representation. The VAE is trained to completion before fitting the diffusion, allowing a consistent latent space during diffusion training. Due to the post-processing property of DP~\citep{dwork2014}, only VAE needs DP training as latent samples enjoy the DP guarantees of the VAE encoder. The following sections cover model architecture and training, fairness intervention, and privacy preservation.

\subsection{Model Architecture}
To learn a curated latent space that encodes heterogeneous tabular data, we first pre-train a transformer-based variational autoencoder (VAE) in two phases. \textit{Phase 1:} optimize for quality; \textit{Phase 2:} Perturb to induce fairness constraints. Next, we train a diffusion model on the encoded training data such that the latent space is projected to a standard Gaussian from which new samples can be drawn. The base architectures of the VAE and the diffusion model are adapted from the TabSyn model~\citep{zhang2024mixedtype}.

\subsubsection{VAE}
\begin{figure}[tb]
    \centering
    \includegraphics[width=\textwidth]{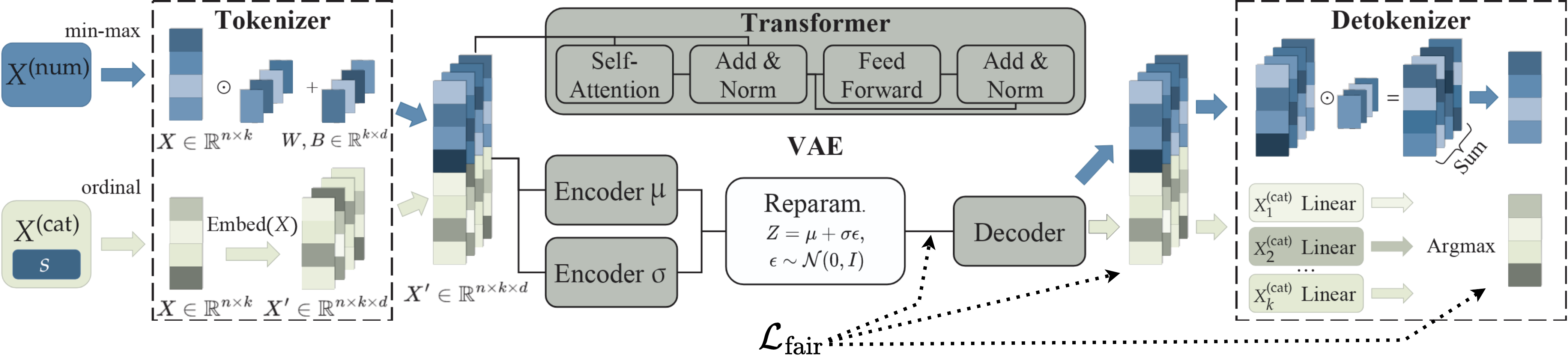}
    \caption{Architecture of base VAE. Phase 1 focuses on quality. Phase 2 performs fairness-informed interventions in the latent space, after decoding, and for the reconstructed distributions of each feature.}
    \label{fig:VAE}
\end{figure}

The VAE, presented in Figure~\ref{fig:VAE}, consists of three main parts: a tokenizer, encoder and decoder transformer blocks, and a detokenizer. The tokenizer's and detokenizer's primary objective is to map the individual attributes to and from a basic feature-wise encoding using categorical embeddings and numerical projections. The core VAE consists of transformer blocks for the encoders and the decoder, enabling a multivariate, learnable representation in latent space.

\subsubsection{Diffusion model} The diffusion has a single task: map the learned VAE latent space to a simple prior. We adopt the TabSyn~\citep{zhang2024mixedtype} score-based diffusion processes, which, in the forward process, perturb the latent vectors with Gaussian noise, while a denoising multi-layer perceptron reconstructs samples. We refer to the work by~\citet{zhang2024mixedtype} for further details. Nevertheless, instead of training on outputs from the encoder $\mu$, we use reparameterized latents using the added noise as regularization~\citep{ramachandranpillai2023fair} and simplifying generation by leveraging learned variance rather than fixed latent scaling.

\subsection{Fairness Intervention}
Given the definition of Disentangled Fairness (see Definition~\ref{def:base_fairness}), the assignment of protected groups to synthetic records is expected to be random. However, to maintain consistency with typical output of a generative model, we include the protected attribute during both training and generation, rather than assigning a value at random. However, since our objective is disentanglement, the reconstruction loss for the protected group attribute should not be of importance. Rather, each protected group should be assigned with similar probabilities regardless of the frequency of occurrence in the training data.

\subsubsection{Balanced training}
To promote a uniform distribution of the protected attribute, we adopt two strategies. First, balanced mini-batch sampling ensures equal group representation. Second, we define a loss encouraging uniform group probabilities for all values of $S$. Let $|S|$ be the number of unique values of $S$, and $\Ssyn \in \mathbb{R}^{n \times |S|}$ be the realized logits. The loss is the $\ell_2$-norm between the mean group-wise softmax distributions and the uniform distribution:
\begin{equation}
    \mathcal{L}_S = \left\lVert\frac{1}{n}\sum_{i=1}^{n}\softmax\left(\Ssyn_{i}\right) - \frac{1}{|S|}\mathbf{1}\right\lVert_{2},
\end{equation}
where $\softmax(X_{i})_{j} = \frac{\exp(X_{ij})}{\sum_{k=1}^{p} \exp(X_{ik})}$ and $\mathbf{1}$ is a vector of ones. In other words, the softmax probability mass for each group should be the same for all groups.

\subsubsection{Disentanglement}
Building on Definition~\ref{def:base_fairness} that defines fairness as the independence of the protected attribute w.r.t.\ the rest of the data, $X \perp\!\!\!\perp S$. We note that this definition is equivalent to 
\begin{equation}
    \Pr(X) = \Pr(X \mid S = s), \quad \forall s \in \mathcal{S},
\end{equation}
implying that the representation of $X$ should be consistent across all protected groups. Thus, the VAE's fairness objective is to ensure similar representations for these groups. To robustly measure representation similarity, we adopt centered kernel alignment (CKA)~\citep{pmlr-v97-kornblith19a}, a normalized adaptation of the Hilbert-Schmidt Independence Criterion (HSIC). Let $A \in \mathbb{R}^{n \times p_0}$ and $B \in \mathbb{R}^{n \times p_1}$ be data matrices with each row representing sample vectors with dimensions $p_0$ and $p_1$, respectively. Let ${K_{i, j}= k(a_i, a_j)}$ and ${L_{i, j} = l(b_i, b_j)}$, where $k$ and $l$ are kernels. The Hilbert-Schmidt criterion is defined as:
\begin{equation}
    \hsic(K, L) = \frac{1}{(n-1)^2}\tr(KHLH),
\end{equation}
where $H$ is the centering matrix~\citep{pmlr-v97-kornblith19a}. Then the normalized version, CKA, is defined as:
\begin{equation}
    \cka(K, L) = \frac{\hsic(K, L)}{\sqrt{\hsic(K, K)\hsic(L, L)}}.
\end{equation}
Assuming that the matrices $A$ and $B$ are centered, the use of a linear kernel simplifies CKA to
\begin{equation}
    \cka(A, B) = \frac{\lVert A\tran B \lVert_{F}^{2}}{\lVert A\tran A \lVert_{F}\lVert B\tran B \lVert_{F}},
\end{equation}
where $\lVert \cdot \rVert_{F}$ is the Frobenius norm~\citep{pmlr-v97-kornblith19a}. While CKA usually compares representations of the same samples, here we compare feature covariance patterns across protected groups. Let ${A, B \in \mathbb{R}^{n_i \times p}}$, where $n_i$ is the number of group samples and $p$ the number of activations. Then, ${\text{CKA}\tran(A, B) := \text{CKA}(A\tran, B\tran)}$ measures similarity in activation patterns.

\subsection{Two-Phase Training}
We divide VAE training into two phases: the first ensures high-quality representations and the second focuses on explicit fairness intervention through disentanglement. This separation allows greater control of bias mitigation on the learned representations. Implicit fairness interventions are applied throughout both phases using two strategies: balanced sampling for equal group representation and a uniform attribute loss for the protected attribute to encourage equal group probabilities. The VAE is trained using DP-SGD across both phases.

\subsubsection{Training Phase 1: Quality}
The first training focuses primarily on ensuring a robust and high-quality representation of the training data. Therefore, this phase consists primarily of standard VAE training with some modifications, which are described later. VAEs are traditionally trained using the ELBO loss~\citep{DBLP:journals/corr/KingmaW13}, however, because we model the prior by training a latent diffusion model, we employ $\beta$-VAE loss, which adds weight $\beta$ for the KL-divergence term~\citep{higgins2017betavae}. Following the approach of \citet{zhang2024mixedtype}, we adaptively decrease the $\beta$-value to allow for better reconstruction. Let $q, p$ denote a probabilistic encoder and decoder, respectively. Assuming conditional independence, we can write the $\beta$-VAE ELBO loss as follows:
\begin{equation}\label{eq:loss_quality}
    \mathcal{L}_\text{ELBO} = \mathbb{E}_{z \sim q(z|x)}\left[\log p(x^{(\text{cat})}|z) + \log p(x^{(\text{num})}|z) \right] - D_{KL}\left(q(z|x) \parallel p(z) \right).
\end{equation}
To approximate the reconstruction loss, we use cross-entropy (CE) for the categorical features and mean squared error (MSE) for the numerical features. With the integration of fairness intervention, the final objective function, with notational flexibility, becomes
\begin{equation}
    \mathcal{L}_\text{quality} = \mathcal{L}_\text{ELBO}(X) + \mathcal{L}_{S}(S)
\end{equation}

\subsubsection{Training Phase 2: Disentanglement}
After learning a strong quality-focused representation, we apply controlled bias mitigation to remove the implicit patterns that connect the data to the protected attribute. Rather than naïvely adding a disentanglement term to $L_\text{quality}$, which may require large, unstable parameter updates, we change the cost landscape to focus on bias mitigation controlled by a penalty term quantifying deviation from a reference representation. Let $M_{\theta_{t}}$ denote the VAE with the learned parameters $\theta$ at time $t$, where $t=0$ represents the timestep after the last parameter update of Phase 1. Similarly, let $q_{\theta_{t}}$ and $p_{\theta_{t}}$ denote the encoder and decoder of $M_{\theta_{t}}$. Considering $M_{\theta_{0}}$ the reference representation, we define the loss function as:
\begin{equation}\label{eq:fair_loss}
    \mathcal{L}_\text{fair} = \underbrace{D\Bigl(q_{\theta_{0}}(z|x) \parallel q_{\theta_{t}}(z|x)\Bigr)}_\text{Divergence penalty} \\
    + \underbrace{\lambda D'\Bigl(q_{\theta_{t}}(z|x_{(0)}) \parallel q_{\theta_{t}}(z|x_{(1)})\Bigr)}_\text{Disentanglement}\;,
\end{equation}
where $D$ and $D'$ may be the same or different divergence measures and $\lambda$ is a hyperparameter controlling the fairness level. To implement fairness interventions during secondary training, we employ a multi-stage protocol that enables precise control over activation pattern perturbations.

Latent Space: The primary objective is to generate an unbiased representation, which makes the latent space a natural candidate for disentanglement. However, in our two-phase training approach, restricting $\mathcal{L}_\text{fair}$ calculation to the latent space limits gradients flow to the encoder, bypassing the decoder and reconstruction processes.

Detokenizer: Fairness interventions post-detokenization enables full gradient flow throughout the network. However, the feature-wise reconstruction of decoded logits requires computing the disentanglement loss per feature, while combining all features would cause disproportionate weighting of features with more unique values. Therefore, we compute the mean loss in feature-wise disentanglement, where the pre-aggregation numerical logits are used. Nevertheless, the detokenizer's limitations in considering joint distributions limits the fairness signal reaching the decoder.

Decoder: In the final intervention stage, the decoder receives direct fairness-aware perturbation. With its output dimensionality matching the latent space, we compute $\mathcal{L}_\text{fair}$ using methods identical to those applied in the latent space.

To approximate $\mathcal{L}_\text{fair}$, we use negative $\cka\tran$ for the disentanglement for all stages. The divergence penalty for the latent and decoder stages is approximated by using the sliced Wasserstein distance (SWD) defined by~\citet{Bonneel2015-os}. At the detokenizer stage, rather than individually computing SWD for all features, we employ $\mathcal{L}_\text{quality}$ as its component $\mathcal{L}_\text{ELBO}$ provides an approximation of the divergence penalty, and $\mathcal{L}_S$ is still required to maintain uniform distribution of the protected attribute.

\subsection{Balanced Sampling Under RDP}
Balanced sampling helps mitigate bias by ensuring equal group representation in training batches. In practice, this can be achieved by selecting the training samples per batch to maintain a uniformly distributed protected attribute. Let $b \leq n,\ b \in \mathbb{Z}^+$ denote the expected batch size. Let $\mathcal{G} := \{S_{(s)} \mid s \in \mathcal{S}\}$ be the set of all disjoint group partitions, and let $m = \min_{g \in \mathcal{G}} |g|$ be the cardinality of the smallest group. Then, the number of iterations per epoch can be defined as
\begin{equation}
    L = \left\lfloor\frac{m}{b} \cdot |\mathcal{G}|\right\rfloor.
\end{equation}
The group-wise sample rates are then found as
\begin{equation}
    \gamma_{(s)} = \frac{m}{L|S_{(s)}|}, \quad \forall s \in \mathcal{S}.
\end{equation}
Computing the sample rates as above causes the omission of some records in each training epoch, but given enough epochs, all records are expected to be sampled. $L$ could be defined with the ceiling function, which would result in each epoch processing more samples than there are records in the training data. However, the impact on the overall training would be negligible, and the only effect on DP-SGD is the number of iterations over which the privacy loss is composed.

Typically, privacy amplification with Poisson sub-sampling in DP-SGD is performed using a fixed sample rate $\gamma$, such that all records are uniformly sampled with probability $\gamma$. However, enforcing balanced batches on imbalanced training data changes the group-wise sample rates, requiring careful consideration to ensure that the DP guarantees hold.

\begin{proposition}\label{prop:rdp_gamma}
    Let $\gamma, \gamma' \in (0, 1]$ denote the sample rate. An algorithm that obeys $(\alpha,\varepsilon(\alpha, \gamma))$-RDP also obeys $(\alpha,\varepsilon(\alpha, \gamma'))$-RDP $\ \forall \gamma' \geq \gamma$.
\end{proposition}
\begin{proof}
    Fix $\ell_2$-sensitivity to 1. The sub-sampled RDP bound is found as~\citep{DBLP:journals/corr/abs-1908-10530}
    \begin{align*}
        \varepsilon &\leq \gamma^2 \frac{2\alpha}{\sigma^2}.
    \end{align*}
    Then the following applies as $\alpha, \sigma > 0$
    \begin{align*}
        \varepsilon &\leq \gamma^2 \frac{2\alpha}{\sigma^2} \leq \gamma'^{\,2} \frac{2\alpha}{\sigma^2}, \quad \forall \gamma' \geq \gamma.
    \end{align*}
    Accordingly, the privacy loss increases with a growing sample rate implying that
    \[
        (\alpha,\varepsilon(\alpha, \gamma))\text{-RDP} \implies (\alpha,\varepsilon(\alpha, \gamma'))\text{-RDP}
    \]
\end{proof}
Proposition~\ref{prop:rdp_gamma} also extends to the general case without a fixed sensitivity defined by~\citet{pmlr-v89-wang19b} since $\gamma$ is always a factor for positive elements. Intuitively, it follows by the proposition that in the worst case, noise should be added with respect to the largest sample rate $\gamma_{\max} = \max\{\gamma_{(s)} \mid s \in \mathcal{S}\}$, i.e., the minimally represented group. However, more importantly, it states that group-wise noise levels can be defined given a fixed $\varepsilon$. Furthermore, it is apparent that the required noise is defined as a combination of $\varepsilon$ or $\gamma$, allowing variable sample rates or privacy budgets. \citet{boenisch2023have} take an approach of varying the privacy budgets across samples, defining individualized DP-SGD (IDP-SGD) based on adjusting the gradient norm clipping for each privacy group.

As DP-SGD adds noise to the accumulated gradients for efficiency, a global noise multiplier, $\sigma_\text{global}$, is required to maintain efficient computations. However, as the noise level is adjusted w.r.t.\ the gradient norm clipping (see Equation~\ref{eq:dp_noise}), $c_\text{global}$, a global noise multiplier can be found if on average the gradient norm clipping corresponds to a pre-determined clipping value. Given Proposition~\ref{prop:rdp_gamma}, our approach with varying sample rates can be formulated as a special case of IDP-SGD, with group-wise sample rates rather than privacy budgets.

In practice, the group-wise noise levels, $\sigma_p$ for group $p$, can be empirically found without $c_\text{global}$ using target values for $\varepsilon$, $\delta$, and $\gamma$, and total number of iterations~\citep{DBLP:journals/corr/abs-2109-12298}. Let $c_p, \sigma_p$ denote the gradient norm clipping and noise multiplier for group $p$, respectively. Let $N$ denote the total number of training samples. Then, because the expected gradient clipping can be described as an average $c_\text{global}=\frac{1}{N}\sum_{p = 1}^{|\mathcal{G}|}\mathcal{G}_{p}c_{p}$, $\sigma_\text{global}$ can be found as
\begin{equation}\label{eq:group_dp_noise_scale}
    \sigma_\text{global} = \left( \frac{1}{N}\sum_{p=1}^{|\mathcal{G}|}\frac{|\mathcal{G}_{p}|}{\sigma_{p}}\right)^{-1}.
\end{equation}
For the derivation of Equation~\ref{eq:group_dp_noise_scale}, we refer to the work by~\citet{boenisch2023have}.

\section{Experiments}\label{sec:experiments}
To evaluate the proposed method, we perform a series of experiments, evaluating data quality, privacy preservation, and bias mitigation on a variety of measures. We explore the behavior of FLIP with changing parameter settings, compare our results with existing methods, and explore how FLIP extends and complements the state of the art in privacy-fairness synthetic data generation. To ensure robust results, all experiments are performed using 3-fold cross-validation, keeping the splits consistent for all methods, thus ensuring comparable results evaluated against the same unseen test sets.

\subsection{Datasets} 
We select four publicly available datasets that recent studies have identified as suitable for fairness assessment \citep{le2022survey}. Specifically, we evaluate on the Adult~\citep{adult_2} and Dutch census datasets~\citep{Van_Der_Laan2000-mx}, as well as two states from the ACS-I (new Adult) dataset~\citep{ding2021-acsi} with high positive target rates between groups, namely Alabama (AL) and Utah (UT). Table~\ref{tab:datasets} contains the relevant descriptors for each dataset.

\begin{table}[tb]
    \centering
    \caption{Dataset used for experiments with relevant descriptors.}
    \label{tab:datasets}
    \begin{tabularx}{0.8\textwidth}{lYYYY}
        \toprule
        \textbf{Dataset} & Samples & \#Num. & \#Cat. & Protected \\
        \midrule
        \textbf{Adult} & 43,914 & 6 & 8 & Sex \\
        \textbf{ACS-I AL} & 20,788 & 2 & 8 & Sex \\
        \textbf{ACS-I UT} & 16,221 & 2 & 8 & Sex \\
        \textbf{Dutch} & 18,440 & 0 & 12 & Sex \\
        \bottomrule
    \end{tabularx}
\end{table}

\subsection{Baseline Competitors} 
We compare our method with \tabfairgan\citep{make4020022}, \tabARGN\citep{tiwald2025tabularargn}, \prefair\citep{Pujol2023}, and \fldgm~\citep{ramachandranpillai2023fair}. These methods are well-established generative models for tabular data, taking fairness and privacy into account. For a more detailed description of the methods, see Section~\ref{sec:related_Work} and Table~\ref{tab:related_work_comparison}.

\subsection{Evaluation measures}
For downstream (fairness) evaluation, we use a state-of-the-art gradient boosting model for tabular data, LightGBM \citep{ke2017lightgbm}. We evaluate the generated data by measuring data quality, fairness, and privacy.

\subsubsection{Data quality}
Following \citet{zhang2024mixedtype}, we measure the higher-order metrics of $\alpha\text{-precision}$ and $\beta\text{-recall}$, measuring the fidelity and diversity of the data~\citep{alaa2022-alpha-beta}. For comparative utility evaluation, we measure the downstream AUC ROC classification performance for the standard prediction task associated with each dataset, ensuring comparability with baselines.

\subsubsection{Fairness}\label{sec:experiments:fairness}
As previously discussed, we consider fairness for synthetic data to be a property of the data rather than relying on a predefined downstream task. Therefore, we assess fairness using the \emph{balanced error rate} (BER, see Definition~\ref{def:epsilon-fairness}) and \emph{adversarial normalized cluster balance} (\ancb). We compute an adversarial adaption of NCB to mitigate spurious fairness measures as a result of random assignment of the protected attribute, which would otherwise produce good BER and NCB results with no bias mitigation. Effectively, this approach checks whether the generated features encode underlying unfair patterns in the predictors. As such, the \ancb\ measure offers a more robust evaluation and can uncover subtle correlations that simpler metrics might not detect.

We employ \ancb\ by using a classification model to infer adversarial values for the protected attribute $S$, from all other features; similar to membership inference for $S$. Then, using this adversarial assignment, we compute the normalized cluster balance based on Definition~\ref{def:cluster-fairness}. To perform the required clustering for \ancb, we first encode the heterogeneous data using factorial analysis of mixed-data (FAMD)~\citep{pages2015-famd} for simplicity and then cluster the predictors $\Xsyn$ using a Gaussian mixture model on all FAMD components.

In addition, to ensure comparability with methods that focus on individual tasks, we also consider \emph{task fairness} (TF, see Section~\ref{sec:fair_synth_dat}), by measuring the TF levels across all features (i.e., potential tasks) in the dataset. Naturally, we adopt different fairness measures for categorical and numerical features. For categorical features, which can be either binary or non-ordinal multi-class, we use the notion of $\varepsilon$-fairness with respect to statistical parity, proposed by~\citet{denis2024fairness}. Specifically, for a classifier $f$, unfairness is defined as:
\begin{equation}\label{eq:cat_task_fairness}
U(f) = \max_{k \in [K]} \big( \Pr[f(X) = k \mid S = 1]  \\
- \Pr[f(X) = k \mid S = 0] \big),
\end{equation}
where $K$ is the number of target classes. When the target feature is binary (i.e., $K = 2$), this definition reduces to the standard notion of Statistical Parity.

For continuous features, we measure fairness using the first-order \emph{Wasserstein distance} $W(\cdot, \cdot)$ between the predicted feature distributions for the two sensitive groups. Numerical unfairness $U_\text{n}$ is then defined as:
\begin{equation}\label{eq:num_task_fairness}    
U_\text{n}(f) = W\bigl( f(X_{(S=1)}),\, f(X_{(S=0)}) \bigr).
\end{equation}

A lower value of $U_\text{n}$ indicates higher similarity of the predictive distributions across protected groups, and thus better fairness.

\subsubsection{Privacy}
The primary privacy indicator for this work is differential privacy, given a fixed privacy budget $\varepsilon$. However, other privacy metrics should be computed to obtain a broader perspective of privacy-preserving capabilities beyond differential privacy~\citep{HYRUP2024100608}. Therefore, we use $\epsilon$-identifiability which computes the proportion of real records for which a synthetic data point is its nearest neighbor~\citep{Yoon2020AnonymizationADS-GAN}. Considering the heterogeneity of tabular data, we use Gower's distance~\citep{gower1971} with a feature-wise weighting based on entropy; we adopt the implementation in \emph{SynthEval}~\citep{lautrup2025syntheval} for our experiments.

\section{Results}\label{sec:results}
In this section, we detail experimental findings that address different components of fairness mitigation under differential privacy.

\subsection{Fairness-quality trade-off}
Understanding the interplay between fairness and privacy hyperparameters is essential for optimizing synthetic data generation and interpreting the potential trade-offs. Figure~\ref{fig:heatmap} illustrates this relationship using the Adult dataset as an example. As the fairness parameter $\lambda$ increases, we observe a consistent decline in quality metrics such as AUC, $\alpha$-precision, and $\beta$-recall. This degradation is expected, as stronger disentanglement alters the learned distribution, reducing perceived fidelity to the original data. In fact, Equation~\ref{eq:fair_loss} describes this trade-off, where the expected fidelity degradation explicitly depends on the inherent bias in the non-protected attributes w.r.t.\ to the protected attribute. Although $\alpha$-precision and $\beta$-recall semantically measure fidelity and diversity~\citep{alaa2022-alpha-beta}, they do so only in the context of the original potentially biased data. Consequently, the degradation observed in Figure~\ref{fig:heatmap} should not be interpreted as the true quality of the generated data, but rather as an approximation of the distributional shift incurred by disentanglement of the protected attribute. 

Similarly, the ROC AUC for the synthetic data w.r.t.\ to the original data (measured for the intended target feature) shows a decrease as the fairness parameter $\lambda$ increases. This is also expected behavior when the original data have an inherent bias against the protected attribute, and when comparing the AUC degradation with improvements of statistical parity and equalized odds, the fairness improvements correspond well with the decrease in quality; in particular, when moving from a $\lambda$ value of $4$ to $8$.

Given this perceived quality degradation, the appropriate fairness parameter $\lambda$ can be defined in the context of both fairness improvements of distributional shift. As stated previously, the most appropriate fairness measures for disentanglement fairness are task-agnostic, thus we consider improvements of BER and \ancb. In Figure~\ref{fig:heatmap}, we observe that at a $\lambda$ value of $4$, both BER and \ancb\ plateaus, suggesting that further bias-mitigating perturbation is unnecessary. Furthermore, further increases in $\lambda$ yield noticeably worse quality measures, indicating that $\lambda=4$ could constitute a reasonable hyperparameter for bias mitigation while minimizing the distributional change.

\begin{figure}[tb]
    \centering
    \includegraphics[width=\textwidth]{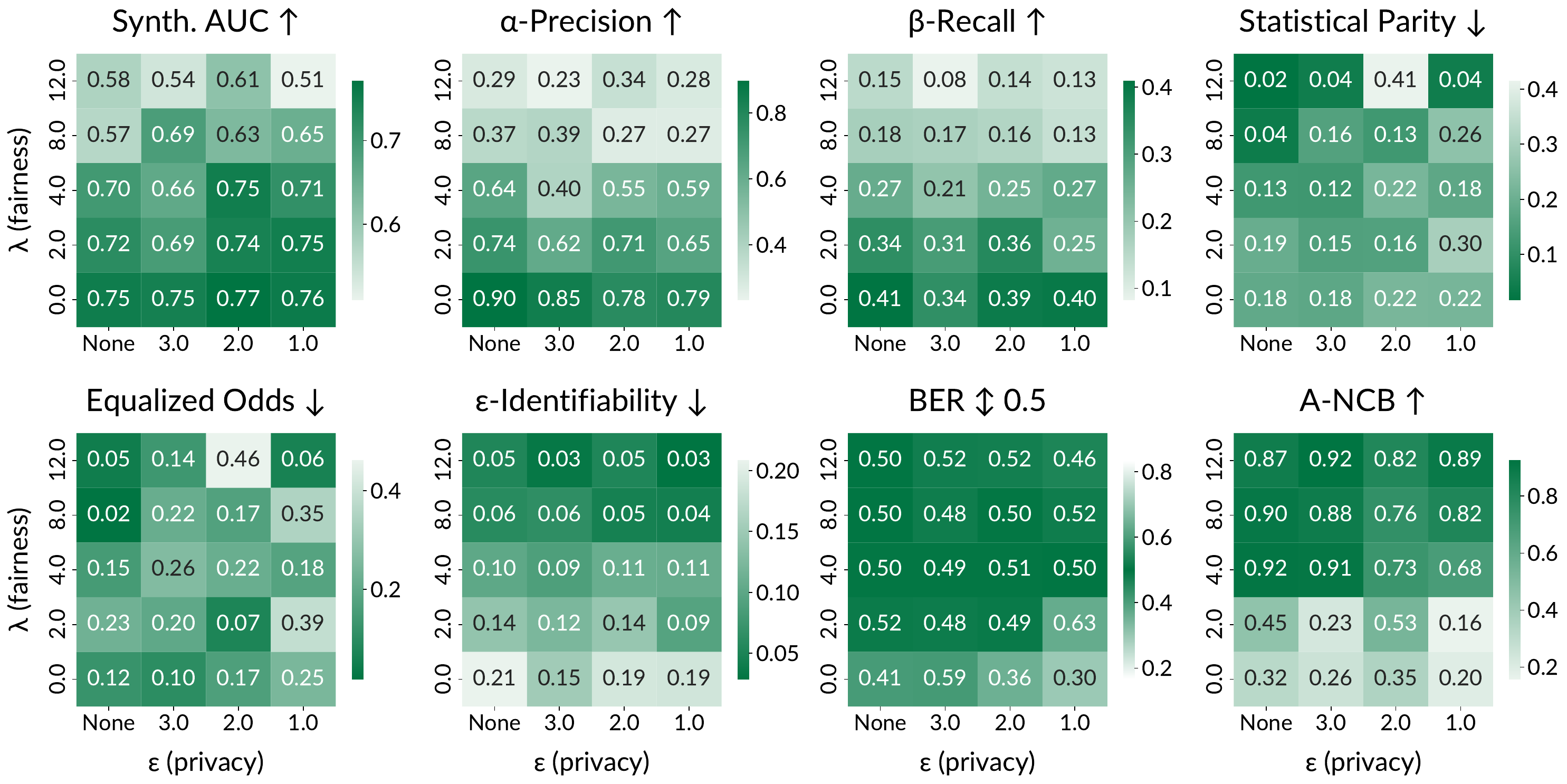}
    \caption{Evaluation results for FLIP on the Adult dataset. Darker shades of green indicate better performance across all metrics. For Synth. AUC, $\alpha$-Precision, $\beta$-Recall, and \ancb, higher values indicate better performance. In contrast, lower scores are preferred for Statistical Parity, Equalized Odds, and $\epsilon$-Identifiability. The optimal value for BER is $0.5$.}
    \label{fig:heatmap}
\end{figure}

\subsection{Privacy-quality trade-off}
The privacy-quality trade-off is a well-known concept that has been repeatedly discussed~\citep{10.1145/3704437, Hyrup2025-zw}. Privacy mitigation generally consists of adding implicit or explicit noise so that identification becomes difficult, resulting in a decrease in the resemblance to the original data. In Figure~\ref{fig:heatmap}, we also observe this effect when considering the similarity-based privacy measure $\epsilon$-identifiability, which improves as the quality measures decrease. Similarly, a more subtle pattern can be observed with stricter privacy budgets for differential privacy, particularly when analyzing $\alpha$-precision, thus adding to the complexity of balancing data quality and privacy protection. Surprisingly, the effect of differential privacy on $\epsilon$-identifiability is negligible, suggesting that while differential privacy effectively limits the contribution of individual data points, it does not significantly alter the overall identifiability landscape for similarity-based privacy definitions. 

\subsection{Fairness-privacy trade-off}
While the trade-offs for both fairness and privacy with respect to data quality are clear from the results in Figure~\ref{fig:heatmap}, the interaction between fairness and privacy is much more nuanced. Notably, we observe clear improvements of $\epsilon$-identifiability with increasing values of $\lambda$. This is not surprising, since the learned representation is intentionally different from the original data distribution, the separation between real and synthetic data decreases, resulting in improved privacy under similarity-based measures.

However, the influence of differential privacy is more complex. Under stronger privacy constraints (lower $\varepsilon$) and weaker fairness regularization, BER and \ancb\ performance slightly deteriorate, suggesting that differential privacy may constrain fairness mitigation by limiting the model’s ability to adjust representations between protected groups. In particular, noise introduced to preserve privacy may obscure subtle patterns necessary for effective fairness-aware disentanglement, thus limiting the model's responsiveness to fairness regularization signals. Analyzing BER and \ancb, we observe that, while stricter privacy constraints limit task-agnostic fairness results, the measures improve, although at a slower rate. Consequently, stronger bias mitigation may be necessary to achieve the same fairness levels under stronger privacy guarantees; naturally, at the cost of degradation in the data quality measures.

\subsection{Comparison with baselines}
Table~\ref{tab:method_comparison} presents comparisons of existing methods and our approach, FLIP, in terms of task-agnostic fairness and the typically reported AUROC for the intended target. Illustrating the aggregated percentage-wise improvements for each method compared to the original data. Here, we find that FLIP significantly outperforms other methods for the task-agnostic measures BER and \ancb\ while suffering a loss in the AUC, yet remaining in a competitive range, indicating that moderate disentanglement retains predictive powers. 

In particular, all methods improve BER, but \tabfairgan and \fldgm\ result in decreases in \ancb, suggesting that at least some of the fairness improvements can be attributed to randomized protected group assignment without successfully obfuscating the underlying biased patterns in the predictors. We also note that we were unable to reproduce the results presented by~\citet{ramachandranpillai2023fair} for their method \fldgm and thus do not make any further comparisons to this method.

\begin{table}
    \centering
    \caption{Relative (percentage \%) changes compared to real data, averaged across all datasets for different hyperparameter settings. Larger values indicate better performance. For each measure, the best performance is emphasized in \textbf{bold}, and the second-best is \underline{underlined}. $\varepsilon=\infty$ indicates no DP for otherwise DP-compatible methods. As \prefair has no implementation without DP, we set $\varepsilon = 100$ as an effectively arbitrarily large value.}
    \label{tab:method_comparison}
    \setlength{\tabcolsep}{1em}
    \begin{tabular}{lccc}
        \toprule
         & BER & \ancb & Synth AUC \\
        \midrule
        FLDGM & +49.12\% & -65.59\% & -31.24\% \\
        PreFair $\epsilon = 3$ & +25.51\% & +4.84\% & -15.58\% \\
        PreFair $\epsilon = \infty$ & +25.59\% & +19.20\% & -13.40\% \\
        TabularARGN $\epsilon = 3$ & +13.18\% & +10.65\% & \underline{-6.07\%} \\
        TabularARGN $\epsilon = \infty$ & +14.33\% & +7.24\% & -6.24\% \\
        TabFairGAN & +7.07\% & -8.85\% & \textbf{-5.36\%} \\
        \cdashline{1-4}\\[-2ex]
        FLIP $\lambda = 4$, $\epsilon = 3$ & \textbf{+67.84\%} & \underline{+20.87\%} & -10.42\% \\
        FLIP $\lambda = 4$, $\epsilon = \infty$ & \underline{+66.76\%} & \textbf{+26.98\%} & -10.46\% \\
        \bottomrule
    \end{tabular}
\end{table}

\subsection{Feature-wise task fairness}
While most fairness-aware works assume a particular target attribute for downstream tasks, FLIP works from a task-agnostic perspective with the objective of lowering the overall bias w.r.t.\ the protected attribute. In this context, we evaluate downstream fairness using every non-protected attribute as a target feature, measuring statistical parity as defined in Equation~\ref{eq:cat_task_fairness} for categorical features and first-order Wasserstein distance defined in Equation~\ref{eq:num_task_fairness} for numerical features. The results for the Adult dataset are presented in Figure~\ref{fig:feature_fairness} for the baselines and various setups of FLIP.

First, we observe a general, though not universal, trend of improved fairness for FLIP as $\lambda$ increases. This suggests that CKA-based disentanglement has the intended effect of reducing bias in the learned representation. While the trend is not perfectly consistent across all settings, the overall direction is encouraging and points to the potential of disentanglement-based approaches for fair synthetic data.

Although the introduction of differential privacy to FLIP noticeably exacerbates the bias-mitigating properties for most settings, the addition generally becomes less impactful for larger $\lambda$ values. It suggests that while fair disentanglement under differential privacy is more challenging, it is not impossible.

Remarkably, FLIP achieves competitive fairness results for the intended target attribute (\emph{income} in the Adult dataset) at higher $\lambda$ values, even without targeted optimization. In addition, FLIP outperforms the baselines in most features, most notably \emph{marital, relationship, age,} and \emph{hours-per-week} with occasional spikes when DP is applied. Generally, the results indicate that while the baselines perform well for the intended target attribute, they have little impact on other features. One exception is \prefair\!\!, which performs well for some non-protected features while producing no improvements for others.

\begin{figure}[tb]
    \centering
    \includegraphics[width=\linewidth]{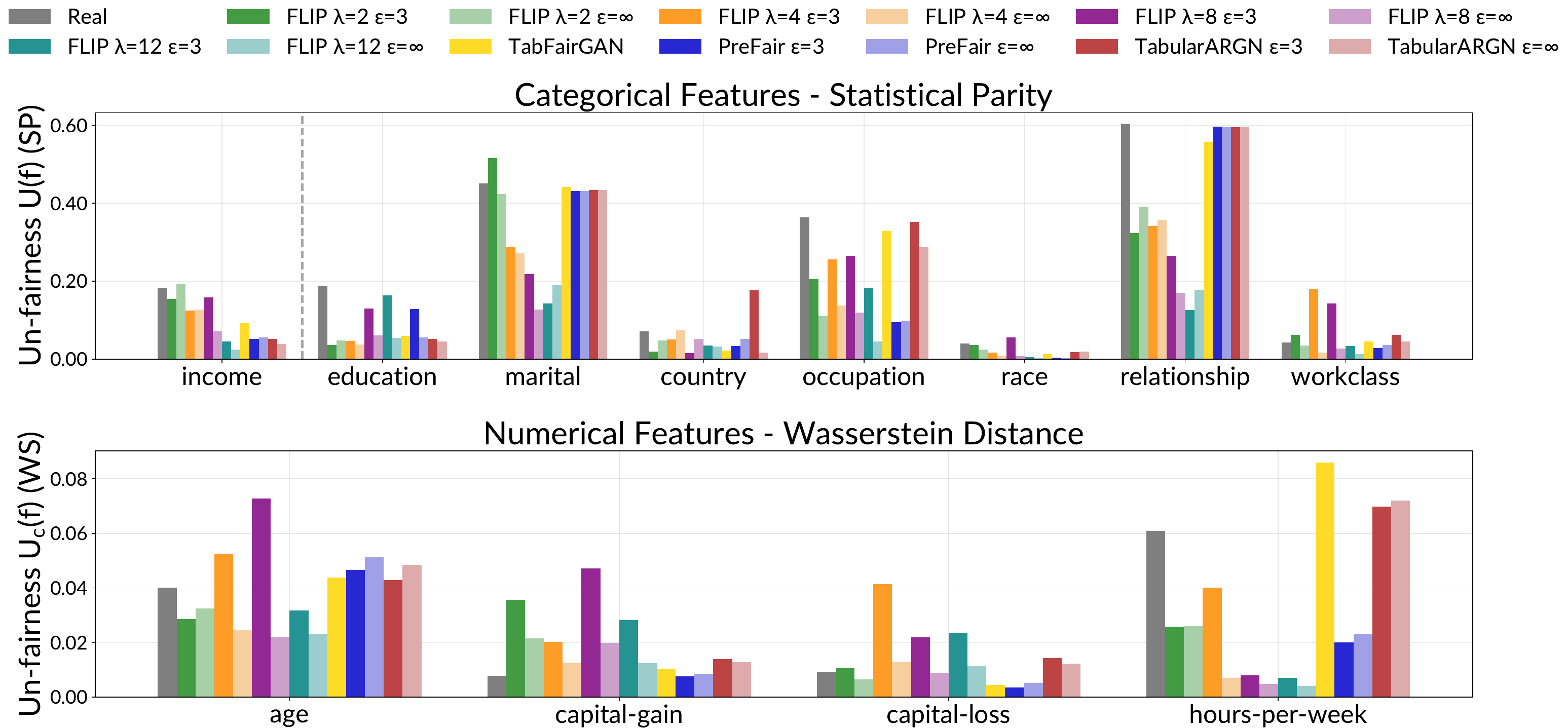}
    \caption{Feature-wise fairness results for FLIP and Baselines on the Adult dataset. Lower values indicate better fairness. Distinct hues represent models while tints denote different $\varepsilon$ values. For FLIP, distinct hues indicate differences in $\lambda$ values. The dashed line separates the intended target from other attributes.}
    \label{fig:feature_fairness}
\end{figure}

\section{Discussion}\label{sec:discussion}
Generating task-agnostic synthetic data that are both fair and private remains a significant challenge. Nevertheless, our results suggest that the disentanglement of the protected attribute is possible with FLIP, providing significant fairness improvements on task-agnostic measures as well as downstream measures. Furthermore, they indicate that while fair disentanglement can enhance distance-based privacy, such as $\epsilon$-identifiability, training under differential privacy may constrain the effectiveness of fairness interventions. However, we observe substantial fairness improvements, particularly in BER and \ancb, with the latter showing the most pronounced gains. Although the results are encouraging, several unresolved challenges remain.

\subsection{Quality evaluation against bias-free data}
A significant challenge with fairness through disentanglement is the evaluation of data quality. Synthetic data are typically evaluated against the original data, since the learning object is to approximate that distribution. In contrast, the learning objective for fairness through disentanglement is to learn a bias-free representation of the population; a distribution that is typically unknown. Consequently, any quality evaluation that compares a disentangled representation with an inherently biased training distribution will produce results with limited meaning with respect to the learning objective.

To address this issue, we assume that for any biased distribution, there exists an unbiased representation for the same population. By extension, it follows that a learned approximation of a biased distribution can be made less biased by systematically perturbing the biased properties embedded within the learned representation. This motivates the two-phase FLIP training procedure, in which a robust and high-quality representation is first learned, followed by a secondary training objective that encourages unbiased neuron activation patterns. The divergence penalty defined in Equation~\ref{eq:fair_loss} ensures that bias mitigation is not simply a result of randomizing the representation, but rather a deliberate and controlled modification of a well-structured representation. The efficacy of this control over disentanglement is evidenced by the gradual degradation of the data quality measures illustrated in Figure~\ref{fig:heatmap} rather than an abrupt drop typically caused by random perturbations.

\subsection{Limitations}
Following the discussion of quality evaluation, it is a natural limitation of this work that appropriate quality measures cannot be performed. Addressing this limitation would require unsupervised quality assessments without reference to the true data or knowledge about a true unbiased distribution, at which point bias mitigation is obsolete. Nevertheless, future research should consider this key evaluation challenge to ensure robust disentanglement.

In addition, a key limitation of attribute disentanglement is the risk of producing out-of-distribution samples, such as \emph{male wives}, which can negatively affect metrics such as $\alpha$-precision and $\beta$-recall. These artifacts arise from the perturbation of the original correlations in the data. While such issues may be mitigated through feature selection or preprocessing strategies tailored to the protected attribute, we leave this for future work focused on integrating fair preprocessing with generative modeling. In addition, one could argue that it is the responsibility of the user, rather than the model design, to remove deterministic associations with the protected attribute. This is particularly relevant when out-of-distribution synthetic records are unwanted in settings focused on disentanglement-based fairness.

Lastly, we observe that incorporating differential privacy can reduce the effectiveness of the bias-mitigating training phase, often requiring stronger perturbations to achieve fair representations. This observation is consistent with previous findings by~\citet{ganev2022-robin-hood}. Given the importance of ensuring both privacy and fairness in data-driven systems, this intersection presents a valuable direction for future research. Specifically, future research is needed to better understand and mitigate this trade-off, with the objective of enhancing the robustness of generative models in applications where privacy and fairness are both critical.

\section{Conclusion}
In this work, we present Fair Latent Intervention under Privacy guarantees (FLIP), a generative model for heterogeneous tabular data that ensures fair and private synthetic data. By defining optimal fairness as the complete disentanglement of the protected attribute from the training data, we adopt a task-agnostic approach that contrasts with most prior work focused on specific downstream tasks. 

In our two-phased training procedure, we first learn a high-quality representation of the true data, followed by a secondary phase that aligns neuron activation patterns across protected groups. Both phases are trained under Rényi differential privacy, which we tailor for group-wise noise levels under a balanced Poisson sampling scheme.

Empirical results show that attribute disentanglement improves task-agnostic fairness measures such as balanced error rate (BER) and adversarial normalized cluster balance (\ancb), while also enhancing downstream fairness across multiple attributes. Although differential privacy introduces limitations to the disentanglement objective, we find that stronger bias mitigation still leads to meaningful fairness improvements. 

In general, our findings support the viability of task-agnostic fairness in synthetic data under privacy constraints, broadening the applicability of synthetic data in sensitive domains.

\bibliography{references.bib}
\bibliographystyle{tmlr}

\end{document}